\documentclass[10pt,conference,a4paper]{IEEEtran}
\ifCLASSOPTIONcompsoc
  \usepackage[nocompress]{cite}
\else
  \usepackage{cite}
\fi
\ifCLASSINFOpdf
\else
\fi
\usepackage{todonotes}

\usepackage{multirow}
\usepackage[inline]{enumitem}
\usepackage{arydshln}
\usepackage{booktabs}
\usepackage{colortbl}

\makeatletter
\def\blfootnote{\xdef\@thefnmark{}\@footnotetext}
\makeatother

\usepackage{bm}
\usepackage{xspace}
\usepackage{mathrsfs}
\usepackage{mathtools}

\usepackage{amsmath}
\usepackage{amssymb}
\usepackage{amsthm}

\newtheorem{proposition}{Proposition}
\newtheorem*{proposition*}{Proposition}

\newcommand\ie{i.e.\xspace}
\newcommand\eg{e.g.\xspace}
\newcommand\Reals{\mathbb{R}}
\DeclareMathOperator{\relu}{ReLU}
\DeclareMathOperator{\identity}{id}
\DeclareMathOperator{\lrelu}{LReLU}
\DeclareMathOperator{\prelu}{PReLU}
\DeclareMathOperator{\elu}{ELU}
\DeclareMathOperator{\selu}{SELU}
\DeclareMathOperator{\softmax}{softmax}
\newcommand\vect[1]{{\boldsymbol{#1}}}

\newcommand\Set[1]{{\bm{\mathcal{#1}}}}
\newcommand\VectorSpace[1]{\bm{\uppercase{#1}}}
\newcommand\LinearMaps[2]{\mathrm{Lin}(#1, #2)}
\newcommand\Translations[1]{\mathrm{K}(#1)}
\DeclareMathOperator{\ConvexHull}{conv}
\DeclareMathOperator{\AffineHull}{aff}
\DeclareMathOperator{\ConvexCone}{cone}
\newcommand\muacro[1]{\texttt{#1}} 
\newcommand\Minus{\textrm{-}}

\newcommand\pp{\muacro{pp}\xspace}

\newcommand\lenet{\muacro{LeNet-5}\xspace}
\newcommand\resnet{\muacro{ResNet-56}\xspace}
\newcommand\kerasnet{\muacro{KerasNet}\xspace}
\newcommand\alexnet{\muacro{AlexNet}\xspace}
\newcommand\mnist{\muacro{MNIST}\xspace}
\newcommand\fashionmnist{\muacro{Fashion-MNIST}\xspace}
\newcommand\cifar{\muacro{CIFAR-10}\xspace}
\newcommand\ilsvrc{\muacro{ILSVRC-2012}\xspace}
\newcommand\rmsprop{\muacro{RMSprop}\xspace}
\newcommand\sgd{\muacro{SGD}\xspace}

\begin{document}
% \bstctlcite{IEEEexample:BSTcontrol}

%
% paper title
% Titles are generally capitalized except for words such as a, an, and, as,
% at, but, by, for, in, nor, of, on, or, the, to and up, which are usually
% not capitalized unless they are the first or last word of the title.
% Linebreaks \\ can be used within to get better formatting as desired.
% Do not put math or special symbols in the title.
\title{Learning Combinations of Activation Functions}

% author names and affiliations
% use a multiple column layout for up to three different
% affiliations
\author{
  \IEEEauthorblockN{Franco Manessi}
\IEEEauthorblockA{Strategic Analytics\\ lastminute.com group\\ Chiasso, Switzerland \\Email: franco.manessi@lastminute.com}
\and
\IEEEauthorblockN{Alessandro Rozza}
\IEEEauthorblockA{Strategic Analytics\\ lastminute.com group\\ Chiasso, Switzerland \\
Email: alessandro.rozza@lastminute.com}
}

% conference papers do not typically use \thanks and this command
% is locked out in conference mode. If really needed, such as for
% the acknowledgment of grants, issue a \IEEEoverridecommandlockouts
% after \documentclass

% for over three affiliations, or if they all won't fit within the width
% of the page (and note that there is less available width in this regard for
% compsoc conferences compared to traditional conferences), use this
% alternative format:
% 
%\author{\IEEEauthorblockN{Michael Shell\IEEEauthorrefmark{1},
%Homer Simpson\IEEEauthorrefmark{2},
%James Kirk\IEEEauthorrefmark{3}, 
%Montgomery Scott\IEEEauthorrefmark{3} and
%Eldon Tyrell\IEEEauthorrefmark{4}}
%\IEEEauthorblockA{\IEEEauthorrefmark{1}School of Electrical and Computer Engineering\\
%Georgia Institute of Technology,
%Atlanta, Georgia 30332--0250\\ Email: see http://www.michaelshell.org/contact.html}
%\IEEEauthorblockA{\IEEEauthorrefmark{2}Twentieth Century Fox, Springfield, USA\\
%Email: homer@thesimpsons.com}
%\IEEEauthorblockA{\IEEEauthorrefmark{3}Starfleet Academy, San Francisco, California 96678-2391\\
%Telephone: (800) 555--1212, Fax: (888) 555--1212}
%\IEEEauthorblockA{\IEEEauthorrefmark{4}Tyrell Inc., 123 Replicant Street, Los Angeles, California 90210--4321}}

% use for special paper notices
%\IEEEspecialpapernotice{(Invited Paper)}

% make the title area
\maketitle

% As a general rule, do not put math, special symbols or citations
% in the abstract
\begin{abstract}
In the last decade, an active area of research has been devoted to design novel activation functions that are able to
help deep neural networks to converge, obtaining better performance. The training procedure of these architectures
usually involves optimization of the weights of their layers only, while non-linearities are generally pre-specified and
their (possible) parameters are usually considered as hyper-parameters to be tuned manually. In this paper, we introduce
two approaches to automatically learn different combinations of base activation functions (such as the identity
function, ReLU, and tanh) during the training phase. We present a thorough comparison of our novel approaches
with well-known architectures (such as LeNet-5, AlexNet, and ResNet-56) on three standard datasets (Fashion-MNIST, CIFAR-10,
and ILSVRC-2012), showing substantial improvements in the overall performance, such as an increase in the top-1 accuracy for
AlexNet on ILSVRC-2012 of 3.01 percentage points.
\end{abstract}

% no keywords

\blfootnote{Published as a conference paper at ICPR 2018.}

% For peer review papers, you can put extra information on the cover
% page as needed:
% \ifCLASSOPTIONpeerreview
% \begin{center} \bfseries EDICS Category: 3-BBND \end{center}
% \fi
%
% For peerreview papers, this IEEEtran command inserts a page break and
% creates the second title. It will be ignored for other modes.
\IEEEpeerreviewmaketitle

\section{Introduction}
In the last decade, deep learning has achieved remarkable results in computer vision, speech recognition and natural language processing, obtaining in some tasks human-like \cite{xiong2016achieving} or even super-human\cite{Ciresan11} performance. The roots of recent successes of deep learning can be found in: 
\begin{enumerate*}[label=(\roman*)]
    \item the increase of the data available for training neural networks, 
    \item the rising of commodity computational power needed to crunch the data,
    \item the development of new techniques, architectures, and activation functions that improve convergence during training of deeper networks, overcoming the obstacle of vanishing/exploding gradient\cite{bengio1994learning, glorot2010understanding}.
\end{enumerate*}

For many years, neural networks have usually employed logistic sigmoid activation functions. Unfortunately, this activation is affected by saturation issues. This problem reduces its effectiveness and, nowadays, its usage in feedforward networks is discouraged~\cite{Goodfellow16}. 
To overcome such weakness and improve accuracy results, an active area of research has been devoted to design novel activation functions. 

The training procedure of these architectures usually involve optimization of the weights of their layers only, while
non-linearities are generally pre-specified and their (possible) parameters are usually considered as hyper-parameters
to be tuned manually. In this paper, we introduce two approaches able to automatically learn combinations of base
activation functions (such as: the identity function, $\relu$, and $\tanh$) during training; our aim is to identify a
search space for the activation functions by means of convex combination or affine combination of the base functions. To
the best of our knowledge, this is one of the first attempts to automatically combine and optimize activation functions
during the training phase.

We tested different well-known networks employing customary activation functions on three standard datasets and we
compared their results with those obtained applying our novel approaches. The techniques proposed in this paper
outperformed the baselines in all the experiments, even when using deep architectures: we found a $3.01$ percentage
points increase in the top-1 accuracy for \alexnet on \ilsvrc.

This paper is organized as follows: in Section~\ref{sec:related} the related works are summarized; in Section~\ref{sec:methods} the proposed methods are
described; in Section~\ref{sec:results} the experimental results are presented; finally, conclusions and future works are summarized in
Section~\ref{sec:conclusion}.

\section{Related Work}\label{sec:related}

Designing activation functions that enable fast training of accurate deep neural networks is an active area of research.
The rectified linear activation function, introduced in \cite{jarrett2009best} and argued in \cite{glorot2011deep} to be
a better biological model than logistic sigmoid activation function, has eased the training of deep neural
networks by alleviating the problems related to weight initialization and vanishing gradient.
Slight variations of $\relu$ have been proposed over the years, such as \emph{leaky ReLU} ($\lrelu$), \cite{Maas2013},
which addresses \emph{dead neuron} issues in $\relu$ networks, \emph{thresholded ReLU}~\cite{konda2014zero},
which tackles the problem of large negative biases in autoencoders, and \emph{parametric ReLU} ($\prelu$)\cite{prelu},
which treats the leakage parameter of $\lrelu$ as a per-filter learnable weight.

A smooth version of $\relu$, called \textit{softplus}, has been proposed in \cite{dugas2001incorporating}. Despite some
theoretical advantages over $\relu$ (it is differentiable everywhere and it has less saturation issues), this activation
function usually does not achieve better results \cite{glorot2011deep} when compared to its basic version.

More recently, \emph{maxout} has been introduced in \cite{goodfellow2013maxout} as an activation function aimed at
enhancing dropout’s abilities as a model averaging technique. Among its extensions, it is worth mentioning the \emph{probabilistic maxout} \cite{Sun2015}, and the
\emph{$L_p$ norm pooling activation} \cite{gulcehre2014learned} that is able to recover the maxout activation when $p
\to \infty$.

Considering the last developments of activation functions in neural networks, it is important to mention the
\emph{exponential linear unit function} ($\elu$)\cite{Clevert2016} and the \emph{scaled exponential linear unit
function} ($\selu$) \cite{klambauer2017self}. Like $\relu$, $\lrelu$, and $\prelu$, $\elu$ reduces the vanishing
gradient problem. Furthermore, $\elu$ has negative values, allowing to push mean unit activations closer to zero
like batch normalization, and speeding up the learning. $\selu$ extends this property ensuring that activations close
to zero mean and unit variance that are propagated through many network layers will converge towards zero mean and unit
variance, even under the presence of noise and perturbations.

With the exception of $\prelu$, all previous activations are pre-specified (\ie non-learnable). A first attempt to
learn activations in a neural network can be found in  \cite{liu1996evolutionary}, where the authors propose to
randomly add or remove logistic or Gaussian activation functions using an evolutionary programming method. On the other
hand, in \cite{poli1999parallel, weingaertner2002hierarchical, turner2013cartesian, khan2013fast} the authors proposed
novel approaches to learn the best activation function per neuron among a pool of allowed activations by means of
genetic and evolutionary algorithms.

A different method has been proposed in \cite{turner2014neuroevolution}, which is able to learn the \emph{hardness}
parameter of a sigmoid function, similarly to the approach employed by $\prelu$ to learn the leakage parameter.

However, all the previous learning techniques are limited by the fact that the family of functions over which the learning takes
place is either finite or a simple parameterization of customary activation functions.

Recently, in \cite{agostinelli2014learning} the authors tackle the problem from a different angle using
piecewise linear activation functions that are learned independently for each neuron using gradient descent. However,
\begin{enumerate*}[label=(\roman*)]
	\item the number of linear pieces is treated as a hyper-parameter of the optimization problem;
	\item the number of learned parameters increases proportionally to the amount of hidden units;
	\item all the learned piecewise linear activation functions are $\relu(x)$ for $x$ large enough (\ie there exists 
		$u\in\Reals$ such that $g(x) = \relu(x)$ for $x \geq u$), thus reducing the expressivity of the learned 
		activation functions by design.
\end{enumerate*}

It is worth mentioning that, in \cite{lin2013network} the authors propose the \emph{network in network} approach where
they replace activation functions in  convolutional layers with small multi-layer perceptrons.

In this paper we try to overcome some of the limitations of the aforementioned approaches. Indeed, the two techniques explained in Section~\ref{sec:methods}:
\begin{enumerate*}[label=(\roman*)]
	\item increase the expressivity power of the learned activation functions by enlarging the hypothesis space explored during training with respect to \cite{liu1996evolutionary, poli1999parallel, weingaertner2002hierarchical, turner2013cartesian, khan2013fast, turner2014neuroevolution};
	\item restrict the hypothesis space with respect to \cite{agostinelli2014learning, lin2013network}, in order to allow a faster training without the need of a careful initialization of network weights (see Proposition~\ref{proposition} and the following lines).
\end{enumerate*}

\section{Methods}\label{sec:methods}

A neural network $N_d$ made of $d$ hidden layers can be seen as the functional composition of $d$ functions $L_i$
followed by a final mapping $\bar L$ that depends on the task at hand (\eg classification, regression): $N_d = \bar L
\circ L_{d} \circ \ldots \circ L_{1}$. In particular, each hidden layer function $L_i$ can be written as the composition
of two functions, $g_i$ followed by $\sigma_i$, the former being a suitable remapping of the layer input neurons, the
latter being the activation function of the layer: $L_i = \sigma_i \circ g_i$. In the most general case, both $\sigma_i$
and $g_i$ are parameterized and belongs to some hypothesis spaces $\Set{H}_{\sigma_i}$ and $\Set{H}_{g_i}$,
respectively. Hence, the learning procedure of $L_i$ amounts to an optimization problem over the layer hypothesis space
$\Set{H}_i = \Set{H}_{\sigma_i} \times \Set{H}_{g_i}$.

Usually, $\sigma_i$ is taken as a non-learnable function; therefore, in the most common scenario $\Set{H}_{\sigma_i}$ is
a singleton: $\Set{H}_i = \{\sigma_i\} \times \Set{H}_{g_i}$. For example, for a  fully-connected layer from
$\Reals^{n_i}$ to $\Reals^{m_i}$ with $\relu$ activation we have that $\Set{H}_{g_i}$ is the set of all \emph{affine
transformations} from $\Reals^{n_i}$ to $\Reals^{m_i}$, \ie $\Set{H}_i = \{ \relu \} \times
\LinearMaps{\Reals^{n_i}}{\Reals^{m_i}} \times \Translations{\Reals^{m_i}}$, where
$\LinearMaps{\VectorSpace{A}}{\VectorSpace{B}}$ and $\Translations{\VectorSpace{B}}$ are the sets of linear
maps between $\VectorSpace{A}$ and $\VectorSpace{B}$, and the set of translations of $\VectorSpace{B}$, respectively.

In this paper, we introduce two techniques to define learnable activation functions that could be plugged in all hidden layers of a neural network architecture. The two approaches differ in how they define the
hypothesis space $\Set{H}_{\sigma_i}$. Both of them are based on the following idea:
\begin{enumerate*}[label=(\roman*)]
    \item șelect a finite set of activation functions $\Set{F} \coloneqq \{ f_1, \ldots, f_N \}$, whose elements  will 
    	be used as base elements;
    \item define the learnable activation function $\sigma_i$ as a linear combination of the elements of $\Set{F}$;
    \item identify a suitable hypothesis space $\Set{H}_{\sigma_i}$;
    \item optimize the whole network, where the hypothesis space of each hidden layer is $\Set{H}_i = \Set{H}_{\sigma_i} 
    	\times \Set{H}_{g_i}$.
\end{enumerate*}

We will now give some basic definitions used throughout the paper. Note that, hereinafter, all
activation functions from $\Reals$ to $\Reals$ will naturally extend to functions from $\Reals^n$ to $\Reals^n$ by means
of entrywise application.

Given a vector space $\VectorSpace{V}$ and a finite subset  $\Set{A}\subseteq\VectorSpace{V}$, we can define the
following two subsets of  $\VectorSpace{V}$:
\begin{enumerate}[noitemsep,topsep=0pt,label=(\roman*)]
	\item the \emph{convex hull} of $\Set{A}$, namely:
		\begin{equation*}
			\ConvexHull(\Set{A}) \coloneqq {\textstyle \{\sum_i c_i \vect{a}_i \mid \sum_i c_i = 1, 
					c_i \geq 0, \vect{a}_i \in \Set{A} \}};
		\end{equation*}
	\item the \emph{affine hull} of $\Set{A}$, namely:
		\begin{equation*}
			\AffineHull(\Set{A}) \coloneqq {\textstyle \{ \sum_i c_i \vect{a}_i \mid \sum_i c_i = 1, 
					\vect{a}_i \in \Set{A} \}}.
		\end{equation*}
\end{enumerate}
We remark that, neither $\ConvexHull(\Set{A})$ nor $\AffineHull(\Set{A})$ are vector subspaces of $\VectorSpace{V}$.
Indeed, $\ConvexHull(\Set{A})$ is just a generic convex subset in $\VectorSpace{V}$ reducing to a $(|\Set{A}| -
1)$-dimensional simplex whenever the elements of $\Set{A}$ are linearly independent. On the other hand,
$\AffineHull(\Set{A})$ is an \emph{affine subspace} of $\VectorSpace{V}$ of dimension $|\Set{A}| - 1$, \ie for an
arbitrary $\bar{\vect{a}} \in \AffineHull(\Set{A})$ the set $\{\vect{a} - \bar{\vect{a}} \mid
\vect{a}\in\AffineHull(\Set{A})\}$ is a \emph{linear subspace} of $\VectorSpace{V}$ of dimension $|\Set{A}| - 1$.
Clearly, $\ConvexHull(\Set{A})  \subset \AffineHull(\Set{A})$.

Let $\Set{F} \coloneqq \{ f_0, f_1, \ldots, f_N\}$ be a finite collection of activation functions $f_i$ from $\Reals$ to
$\Reals$. We can define a vector space $\VectorSpace{F}$ from $\Set{F}$ by taking all linear combinations $\sum_i c_i
f_i$. Note that, despite $\Set{F}$ is (by definition) a spanning set of $\VectorSpace{F}$, it is not generally a basis;
indeed $|\Set{F}| \geq \dim\VectorSpace{F}$.

Since (almost everywhere) differentiability is a property preserved by finite linear combinations, and since
$\ConvexHull(\Set{F}) \subset \AffineHull(\Set{F}) \subseteq \VectorSpace{F}$, assuming that $\Set{F}$ contains (almost
everywhere) differentiable activation functions, $\ConvexHull(\Set{F})$ and $\AffineHull(\Set{F})$ are made of (almost
everywhere) differentiable functions, \ie valid activation functions for a neural network that can be learned by means
of gradient descent.

The activations used in real world scenarios are usually monotonic increasing functions. Unfortunately, the monotonicity
is not ensured under arbitrary linear combination, meaning that even if all $f_i \in \Set{F}$ are non-decreasing
functions, an arbitrary element $\bar f \in \VectorSpace{F}$ might be neither a non-decreasing nor a non-increasing
function. As a matter of fact, considering only non-decreasing differentiable functions ($f_i' \geq 0$
$\forall f_i \in \Set{F}$), all non-decreasing differentiable functions in $\VectorSpace{F}$ lie inside
the \emph{convex cone} $\ConvexCone(\Set{F}) \subset \VectorSpace{F}$, \ie:
\begin{equation*}
	\ConvexCone(\Set{F}) \coloneqq {\textstyle\{ \sum_i c_i f_i \mid c_i \geq 0, 
		f_i \in \Set{F} \}}.
\end{equation*}
Indeed, $\forall g\in\ConvexCone(\Set{F})$ we have that $g' \geq 0$. Thanks to the definition of $\AffineHull(\Set{F})$,
$\ConvexCone(\Set{F})$, and $\ConvexHull(\Set{F})$, we can conclude that $\ConvexHull(\Set{F}) = \ConvexCone(\Set{F})
\cap \AffineHull(\Set{F})$, which implies that monotonicity of the elements of $\Set{F}$ is preserved by all the
elements of $\ConvexHull(\Set{F})$ but not by $\AffineHull(\Set{F})$ (see Figure~\ref{fig:sets}). Nevertheless,
in~\cite{cybenko1989approximation} it is shown that even non-monotonic activation functions can approximate arbitrarily
complex functions for sufficiently large neural networks. Indeed, in \cite{Goodfellow16} the authors trained a
feedforward network using cosine activation functions on the \mnist dataset obtaining an error rate smaller than $1\%$.
Therefore, also $\AffineHull(\Set{F})$ is a proper candidate for $\Set{H}_{\sigma_i}$.

\begin{figure}[t]
	\includegraphics[width=\columnwidth]{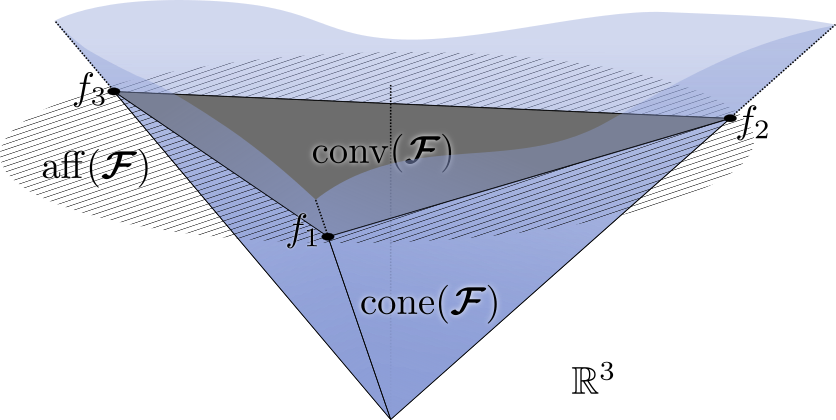}
    \caption{\label{fig:sets}The figure shows the relationship between affine hull, convex hull, and convex cone for a
		set $\Set{F}$ made of three linearly independent elements: $\Set{F} \coloneqq \{f_1, f_2, f_3\}$. Since
		$\dim\VectorSpace{F} = |\Set{F}| = 3$, $\ConvexHull(\Set{F})$ is a 2-simplex (the gray triangle 
		in the figure), while $ \AffineHull(\Set{F})$ is a plane within the three dimensional vector 
		space $\VectorSpace{F}$ (the plane identified by the dashed circle in the figure).
		$\ConvexCone(\Set{F})$ is the three dimensional manifold delimited by the three incident straight lines $l_i 
		\coloneqq \{ x\in\VectorSpace{F} \mid x = \alpha f_i, \forall 0\leq\alpha\in\Reals\}$ for $i=1, 2, 3$, \ie the 
		cone extremal rays. It can easily be seen that $\ConvexHull(\Set{F})$ corresponds to the intersection of  
		$\ConvexCone(\Set{F})$ with $\AffineHull(\Set{F})$.
		\vspace{-1em}
	}
\end{figure}

In addition, the affine subspace $\AffineHull(\Set{F})$ enjoys the following property.
\begin{proposition}\label{proposition}
	Let all $f_i\in\Set{F}$ be linearly independent and approximate the identity  function near origin (\ie $f_i(0) = 0$ 
	and $f_i'(0) = 1$), then, $\bar f \in \VectorSpace{F}$ approximates the identity if and only if  $\bar f \in
	\AffineHull(\Set{F})$.
\end{proposition}
\begin{proof}
	The proof is immediate. Let us expand $\VectorSpace{F} \ni \bar f$ as $\bar f = \sum_i c_i f_i$ (by hypothesis the 
	$f_i$ form a basis) and prove the two-way implication.
	\begin{description}[leftmargin=0cm, topsep=0pt]
		\item[($\Rightarrow$)] By hypothesis, $\bar f(0) = 0$ and $\bar f'(0) = 1$. The relation on the derivative 
			reads $\sum_i c_i f'_i(0) = \sum_i c_i = 1$. In turns, this implies that $\bar f = \sum_i c_i f_i$ with
			 $\sum_i c_i = 1$, namely $\bar f \in \AffineHull(\Set{F})$.
		\item[($\Leftarrow$)] By hypothesis, $\bar f = \sum_i c_i f_i$ with $\sum_i c_i = 1$. Hence, $\bar f(0) = 
			\sum_i c_i f_i(0) = 0$ and $\bar f'(0) = \sum_i c_i f'_i(0) = \sum_i c_i = 1$, namely $\bar f$ approximates 
			the identity near origin.
	\end{description}
	\vspace{-1em}
\end{proof}

Since $\ConvexHull(\Set{F}) \subset \AffineHull(\Set{F})$, also $\ConvexHull(\Set{F})$ enjoys the same property. It is
important to underline that, activation functions approximating the identity near origin, have been argued to be
desirable to train deep networks randomly initialized with weights near zero \cite{identity}. Indeed, for such an
activation function the initial gradients will be relatively large, thus speeding up the training phase. Note that, such
behavior is enjoyed also by $\relu$, which approximates the identity only from one side: \ie $x \to {0^+}$. Furthermore,
training deep networks with activation functions not approximating the identity near origin requires a more careful
initialization, in order to properly scale the layers input \cite{sussillo2014random}.

For the aforementioned reasons:
\begin{enumerate*}[label=(\roman*)]
	\item preserved differentiability;
	\item absence of the requirement of monotonicity;
	\item approximation of the identity near origin;
\end{enumerate*}
both $\AffineHull(\Set{F})$ and $\ConvexHull(\Set{F})$ are good candidates for 
$\Set{H}_{\sigma_i}$. 

Thanks to the previous definitions and the aforementioned properties, we can now formalize the two techniques to build
learnable activation functions as follows:
\begin{enumerate}[noitemsep,topsep=0pt,label=(\roman*)]
	\item \label{enum:requirements}choose a finite set $\Set{F} = \{f_1, \ldots, f_N \}$, where each $f_i$ is a (almost
		everywhere) differentiable activation function approximating the identity near origin (at least from one side);
	\item define a new activation function $\bar f$ as a linear combination of all the $f_i \in \Set{F}$;
	\item select as hypothesis space $\Set{H}_{\bar f}$ the sets $\AffineHull(\Set{F})$ or $\ConvexHull(\Set{F})$.
\end{enumerate}

In Section~\ref{sec:results}, we present results using the following choices for $\Set{F}$: 
\begin{equation}\label{eq:chosen-set-f}
	\begin{aligned}
		&\Set{F} \coloneqq \{ \identity, \relu \}, &&
		\Set{F} \coloneqq \{ \identity, \tanh \}, \\
		&\Set{F} \coloneqq \{ \relu, \tanh \}, &&
		\Set{F} \coloneqq \{ \identity, \relu, \tanh \},
	\end{aligned}
\end{equation}
where $\identity$ is the identity function. Clearly, other choices of $\Set{F}$ may be available, provided the requirements
in \ref{enum:requirements} are satisfied.

Since $\ConvexHull(\Set{F}) \subset \AffineHull(\Set{F})$, the convex hull-based technique can to be understood as a
\emph{regularized version} of the affine hull-based one, where the corresponding hypothesis space has been explicitly
constrained to be \emph{compact}. Such a regularization, in addition to restrict the complexity of the hypothesis space,
guarantees that the final activation function is monotonic (provided all $f_i \in \Set{F}$ are monotonic as well).
Moreover, the convex hull-based technique together with $\Set{F} \coloneqq \{\identity, \relu,\}$ recovers the learnable
$\lrelu$ activation function, \ie $\lrelu_\alpha(x) = x$ if $x \geq 0$, while $\lrelu_\alpha(x) = \alpha x$ otherwise,
where $0 \leq \alpha \ll 1$ (usually $\alpha = 10^{\Minus2}$).  Indeed, $\ConvexHull(\Set{F}) = \{ \bar f \coloneqq p
\cdot \identity + (1 - p) \cdot \relu \text{ with } 0\leq p \leq1\}$ and since $\relu = \identity$ for $x\geq0$ and
$\relu = 0$ otherwise, we have that  $\bar f = p \cdot \identity + (1 - p) \cdot \identity = \identity$ for $x\geq0$ and
$\bar f = p \cdot \identity + (1 - p) \cdot 0 = p\cdot\identity$ otherwise, \ie $\lrelu_p$.

\begin{figure}[t]
	\centering
	\includegraphics[width=\columnwidth]{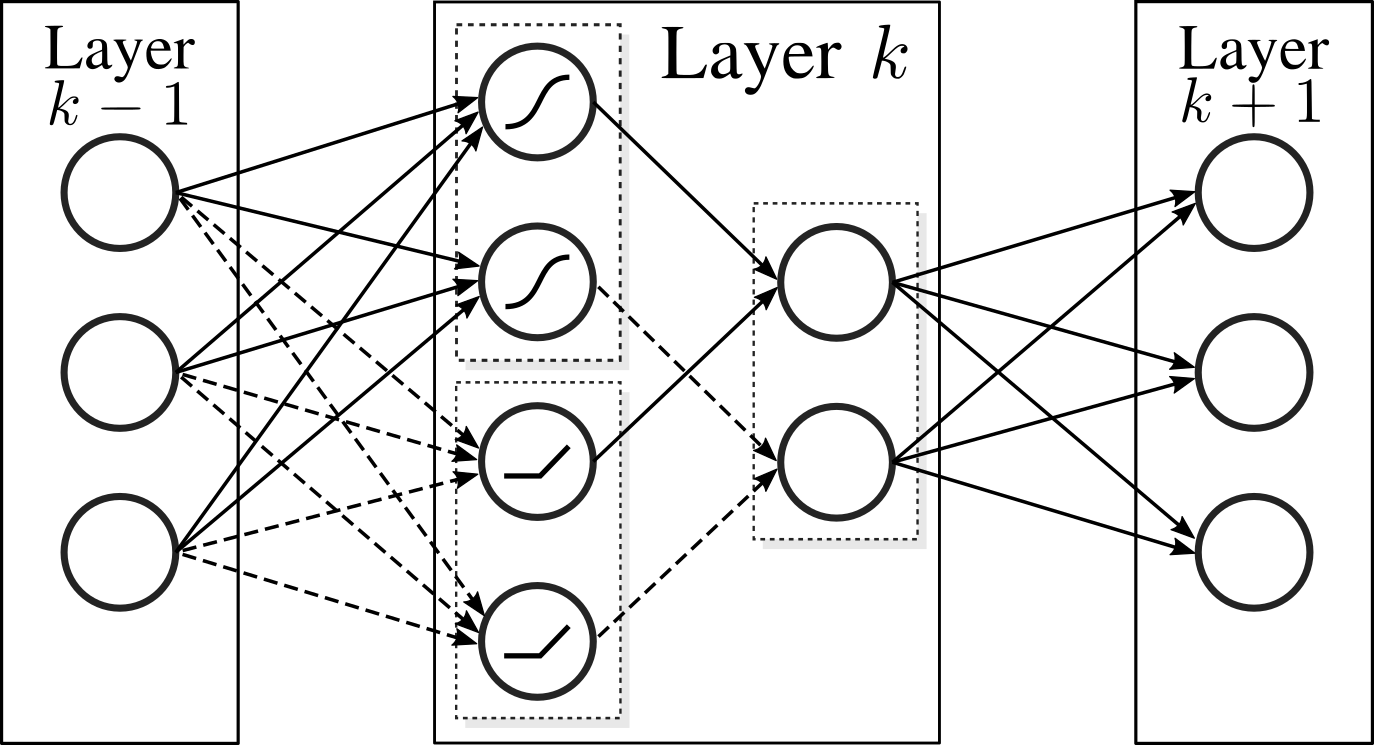}
    \caption{\label{fig:stacked}The figure shows how a neural network layer with convex/affine hull-based activation can 
    	be seen as a two-stage pipeline. Here the $k$-th layer is a two neurons layer, with set $\Set{F}$ composed of 
    	two base functions. The first stage of the pipeline is made of two stacked layers, each one featuring only one 
    	activation belonging to $\Set{F}$. The weights between the two stacked layers are shared. The second stage of
    	the pipeline is a 1D-convolution with kernel of size 1 between the stacked layers, which play the role of $n$ 
    	channels for the convolution. The weights of the convolution are constrained to sum to one, and also to be 
    	positive when using the convex hull-based technique.
		\vspace{-1em}
	}
\end{figure}
It is worth mentioning that, as shown in Figure~\ref{fig:stacked}, layers using convex hull-based and affine hull-based 
activations with $n$ base functions can also be seen as the following two-stage pipeline: 
\begin{enumerate*}[label=(\roman*)]
	\item $n$ stacked fully connected layers, each of them featuring one of the base functions, and all of them sharing 
		weights;
	\item a 1D-convolution with kernel of size 1 between the $n$ stacked layers (which are treated by the convolution as 
		$n$ separate channels), whose weights are constrained to sum to one (and to be positive in case of the convex 
		hull-based technique)\footnote{The 1D-convolution with kernel of size 1 can also be seen as a weighted average 
		of the stacked fully connected layers (with possibly negative weights in case of the affine hull-based
		technique).}.
\end{enumerate*}

\section{Results}\label{sec:results}

\begin{figure*}[t]
   \includegraphics[width=\textwidth]{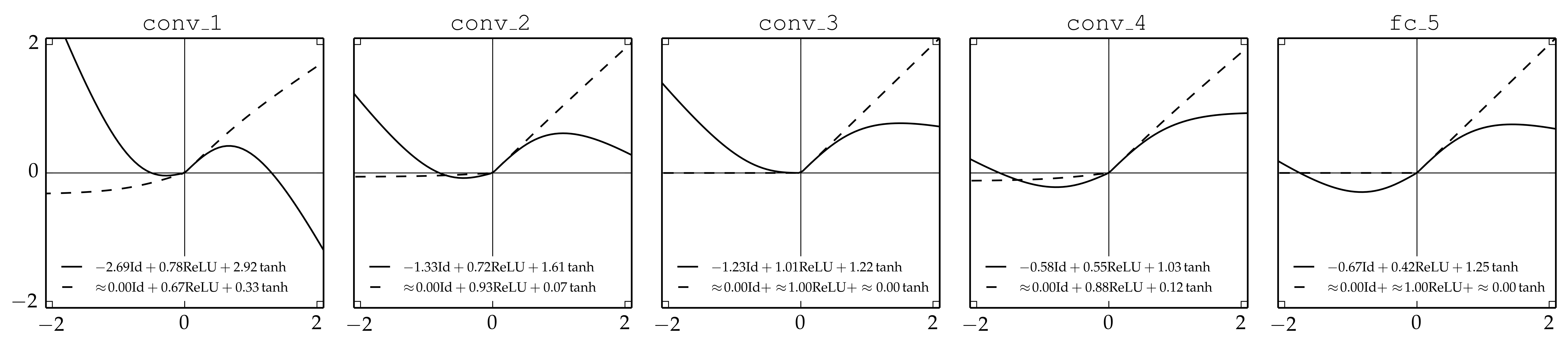}
   % \vspace{-2em}
   \caption{\label{fig:kerasnet-activations}The figure shows the hidden affine hull-based and convex hull-based 
      activation functions learned during the training of the \kerasnet architecture on \cifar dataset. Activations 
      learned on $\AffineHull(\{\identity, \relu, \tanh\})$ are represented with solid lines, while those learned on 
      $\ConvexHull(\{\identity, \relu, \tanh\})$ are dashed.
      % \vspace{-2em}
   }
\end{figure*}

We tested the convex hull-based and the affine hull-based approaches by evaluating their effectiveness on three publicly
available datasets used for image classification, greatly differing in the number of input features and examples.
Moreover, each dataset was evaluated using different network architectures. These networks were trained and tested using
as activation functions (for all their hidden layers) those learned by the convex hull-based and the affine hull-based
approaches combining the base activations reported in Equation~\eqref{eq:chosen-set-f}. In addition, the base
activation functions alone and $\lrelu$ were also employed in order to compare the overall performance.

Specifically, the datasets we used are:
\begin{description}[leftmargin=0cm, topsep=0pt]
   \item[\fashionmnist:] it is a dataset of Zalando's article images, composed by a 
      training and test sets of 60k and 10k examples, respectively. Each example is a 28x28 grayscale image, associated 
      with a label that belongs to 10 classes \cite{xiao2017/online}. We divided each pixel value by 255, and we augmented the
      resulting dataset during training phase by means of random horizontal flip and image shifting;
   \item[\cifar:] it is a balanced dataset of 60k 32x32 colour images belonging to 10 classes. There 
      are 50k training images and 10k test images \cite{krizhevsky2009learning}. Also in this case we divided each pixel 
      value by 255, and we augmented the resulting dataset during training phase by means of random horizontal 
      flip and image shifting;
   \item[\ilsvrc:] a 1k classes classification task with 1.2M training examples and 50k validation examples. The 
      examples are colour images of various sizes \cite{ILSVRC15}. We resized each image to 256x256 pixels, we 
      subtracted pixel mean values, and we augmented the dataset during training phase by randomly cropping to 227x277 
      pixels and horizontally flipping the images. Note that, we did not train our network with the relighting data-augmentation proposed in~\cite{krizhevsky2012imagenet}. 
\end{description}

\begin{table}
  \centering
  \caption{\kerasnet architecture}
  \label{tab:keras-architecture}
  \resizebox{.5\columnwidth}{!}{
     \begin{tabular}{@{}llc@{}}
        \toprule
        Id              & Layers          & Properties       \\
        \midrule
        \multirow{2}{*}{\muacro{conv\_1}} & 2D convolution  & $32\times(3, 3)$ \\
                                          & Activation      &                  \\
        % \cmidrule(lr){2-3}
        \midrule
        \multirow{2}{*}{\muacro{conv\_2}} & 2D convolution  & $32\times(3, 3)$ \\
                                          & Activation      &                  \\
                                          % & \multicolumn{2}{c}{Activation}     \\
        % \cmidrule(lr){2-3}
        \midrule
                                          & Max pooling     & $(2, 2)$         \\
                                          & Dropout         & $25\%$           \\
        % \cmidrule(lr){2-3}
        \midrule
        \multirow{2}{*}{\muacro{conv\_3}} & 2D convolution  & $64\times(3, 3)$ \\
                                          & Activation      &                  \\
                                          % & \multicolumn{2}{c}{Activation}     \\
        % \cmidrule(lr){2-3}
        \midrule
        \multirow{2}{*}{\muacro{conv\_4}} & 2D convolution& $64\times(3, 3)$   \\
                                          & Activation      &                  \\
                                          % & \multicolumn{2}{c}{Activation}     \\
        % \cmidrule(lr){2-3}
        \midrule
                                          & Max pooling     & $(2, 2)$         \\
                                          & Dropout         & $25\%$           \\
        % \cmidrule(lr){2-3}
        \midrule
        \multirow{2}{*}{\muacro{fc\_5}}   & Fully connected & $512$            \\
                                          & Activation      &                  \\
                                          % & \multicolumn{2}{c}{Activation}     \\
        % \cmidrule(lr){2-3}
        \midrule
                                          & Dropout         & $50\%$           \\
        % \cmidrule(lr){2-3}
        \midrule
        \multirow{2}{*}{\muacro{fc\_6}}   & Fully connected & $10$             \\
                                          & Activation      & $\softmax$       \\
                                          % & \multicolumn{2}{c}{$\softmax$}     \\
        \bottomrule
     \end{tabular}
  }
  \vspace{-2em}
\end{table}

The considered architectures are the following:
\begin{description}[leftmargin=0cm, topsep=0pt]
   \item[\lenet:] a convolutional network made of two convolutional layers followed by two fully connected layers 
      \cite{lecun1998gradient}. The convolutional layers have respectively 20 and 50 filters of size $5\times5$, while 
      the hidden fully connected layer is made of 500 neurons. Max pooling with size $2\times2$ is used after each 
      convolutional layer, without dropout. We assessed \lenet on \fashionmnist and \cifar datasets, 
      resulting in two networks with 431k and 657k parameters, respectively;
   \item[\kerasnet:] a convolutional neural network included in the Keras framework \cite{chollet2015keras}. It is made 
      of 4 convolutional layers and 2 fully connected layers, and it employs both max pooling and dropout. This 
      architecture is presented in Table~\ref{tab:keras-architecture}. We tested \kerasnet on both \fashionmnist and 
      \cifar datasets, resulting in two networks with 890k and 1.2M parameters, respectively;
   \item[\resnet:] a residual network made of 56 total layers, employing pooling and skip connection \cite{he2016deep}. 
      Its performance has been evaluated on \cifar, corresponding to a network with 858k parameters;
   \item[\alexnet:] a convolutional network made of 5 convolutional and 3 fully connected layers 
      \cite{krizhevsky2012imagenet}. We tested it against \ilsvrc dataset, resulting in a network with 61M parameters. 
      Note that, as shown in \cite{krizhevsky2012imagenet}, $\relu$-based activation functions significantly outperform
      networks based on other activations. Therefore, in this context $\relu$-based networks only have been built and 
      tested for comparison.
\end{description}
All networks were implemented from scratch using the Keras framework \cite{chollet2015keras} on top of TensorFlow
\cite{tensorflow2015-whitepaper}. Notice that, our \alexnet implementation is a porting to Keras of the \emph{Caffe} architecture\footnote{See {https://github.com/BVLC/caffe/tree/master/models/bvlc\_alexnet}.}. We trained \lenet, \kerasnet, and \resnet using \rmsprop with learning 
rate $10^{\Minus4}$ and learning rate decay over each mini-batch update of $10^{\Minus6}$. For \alexnet we used \sgd with starting learning rate $10^{\Minus2}$, step-wise learning rate decay, weight-decay 
hyper-parameter $5\cdot10^{\Minus4}$, and momentum $0.9$.

\begin{table*}
   \centering
   \caption{Experiment results.
      The table shows the top-1 accuracy results for all the analyzed networks on \fashionmnist test set, \cifar test set, and 
      \ilsvrc validation set. Convex hull-based and affine hull-based activations achieving top-1 accuracy results greater 
      than their corresponding base activation functions are highlighted in boldface. The best result for each 
      network/dataset is shaded.
   }
   \label{tab:results}
   % \resizebox{.8\textwidth}{!}{
   \begin{tabular}{@{}llcclccclc@{}}
      \toprule
                                                 && \multicolumn{2}{c}{\fashionmnist}      & & \multicolumn{3}{c}{\cifar}                          & & \ilsvrc    \\
      \cmidrule(l){3-4}\cmidrule(l){6-8}\cmidrule(l){10-10}
      Activation                                 && \lenet             & \kerasnet         & & \lenet             & \kerasnet          & \resnet   & & \alexnet   \\
      \midrule
      $\identity$                                && $90.50\%$          & $90.51\%$         & & $75.27\%$          & $75.23\%$          & $42.60\%$ & & ---        \\ 
      $\relu$                                    && $91.06\%$          & $90.79\%$         & & $80.37\%$          & $79.97\%$          & $89.18\%$ & & $56.75\%$  \\ 
      $\tanh$                                    && $92.33\%$          & $93.43\%$         & & $78.96\%$          & $82.86\%$          & $82.65\%$ & & ---        \\
      \cmidrule(r){1-1}\cmidrule(l){3-4}\cmidrule(l){6-8}\cmidrule(l){10-10}
      $\lrelu$                                    && $91.03\%$          & $91.13\%$         & & $80.94\%$          & $80.07\%$          & $89.38\%$ & & $57.03\%$        \\
      \cmidrule(r){1-1}\cmidrule(l){3-4}\cmidrule(l){6-8}\cmidrule(l){10-10}
      $\ConvexHull(\{\identity, \relu\})$        && $\bm{91.87\%}$     & $\bm{92.39\%}$  & & $\bm{80.94\%}$ & $\bm{84.74\%}$ & \cellcolor{black!25}$\bm{90.51\%}$     & & $\bm{57.54\%}$  \\ 
      $\ConvexHull(\{\identity, \tanh\})$        && $\bm{92.36\%}$     & $\bm{93.64\%}$   & & $\bm{79.45\%}$ & ${78.53\%}$ & $\bm{86.46\%}$       & & ---  \\
      $\ConvexHull(\{\tanh, \relu\})$            && $\bm{92.56\%}$     & $92.04\%$        & & $80.21\%$ & $\bm{84.80\%}$ & ${88.31\%}$     & & $\bm{58.13\%}$  \\
      $\ConvexHull(\{\identity, \relu, \tanh\})$ && $92.21\%$          & $92.94\%$         & & $\bm{80.48\%}$ & $\bm{85.21\%}$ & $\bm{89.60\%}$     & & $56.55\%$  \\   
      \cmidrule(r){1-1}\cmidrule(l){3-4}\cmidrule(l){6-8}\cmidrule(l){10-10}
      $\AffineHull(\{\identity, \relu\})$        && $\bm{92.83\%}$     & $\bm{93.37\%}$    & & $\bm{80.52\%}$ & $\bm{84.93\%}$ & $88.92\%$         & & $\bm{58.48\%}$  \\
      $\AffineHull(\{\identity, \tanh\})$        && $\bm{92.65\%}$     & \cellcolor{black!25}$\bm{94.41\%}$    & & $\bm{78.97\%}$ & $\bm{86.05\%}$ & $\bm{82.97\%}$    & & --- \\   
      $\AffineHull(\{\tanh, \relu\})$            && \cellcolor{black!25}$\bm{93.02\%}$     & $\bm{93.48\%}$    & & \cellcolor{black!25}$\bm{81.23\%}$ & $\bm{84.83\%}$ & $\bm{89.44\%}$    & & \cellcolor{black!25}$\bm{60.04\%}$  \\
      $\AffineHull(\{\identity, \relu, \tanh\})$ && $\bm{92.80\%}$     & \cellcolor{black!25}$\bm{94.41\%}$    & & $80.13\%$ & \cellcolor{black!25}$\bm{87.45\%}$ & $88.62\%$         & & $\bm{57.20\%}$  \\
      \bottomrule
   \end{tabular}
   % }
   \vspace{-2em}
\end{table*}

Table~\ref{tab:results} shows the top-1 accuracy for all the run experiments. The best configurations (shaded cells in
the table) are always achieved using our techniques, where in 5 out of 6 experiments the affine hull-based approach outperformed convex hull-based ones. The
uplift in top-1 accuracy using our approaches compared to customary activations goes from $0.69$
percentage points (\pp) for \lenet on \fashionmnist up to $4.59$ \pp for \kerasnet on \cifar. 
It is worth mentioning that, even in deep neural networks, such as \alexnet, a substantial increase is observed, \ie $3.01$ \pp on \ilsvrc.

Moreover, the proposed techniques usually
achieve better results than their corresponding base activation functions (boldface in the table). Note that, the novel
activations work well for very deep networks and also with various architectures involving different types of layer
(\eg residual unit, dropout, pooling, convolutional, and fully connected).

Furthermore, our experiments show how the learning of the leakage parameter achieved by the
activation based on $\ConvexHull(\{\identity, \relu\})$ allows to outperform or to achieve the same results of $\lrelu$.

Figure~\ref{fig:kerasnet-activations} shows activations learned by \kerasnet on \cifar when using convex hull-based
and affine hull-based activations with $\Set{F} = \{ \identity, \relu, \tanh \}$. It is possible to notice that, as
already theoretically proved in Section~\ref{sec:methods}, convex combinations preserved monotonicity of the base
activation functions while affine ones did not. 

\section{Conclusion}\label{sec:conclusion}

In this paper we introduced two novel techniques able to learn new activations starting from a finite collection
$\Set{F}$ of base functions. Both our ideas are based on building an arbitrary linear combination of the elements of
$\Set{F}$ and on defining a suitable hypothesis space where the learning procedure of the linear combination takes
place. The hypothesis spaces for the two techniques are $\ConvexHull(\Set{F})$ and $\AffineHull(\Set{F})$. We showed
that, provided all the elements of $\Set{F}$ approximate the identity near origin, $\AffineHull(\Set{F})$ is the only
set where it is possible to find combined activations that also approximate $\identity$ near origin. Moreover, $\AffineHull(\Set{F})$
allows to explore non-monotonic activation functions, while $\ConvexHull(\{ \identity, \relu \})$ may be seen as a learnable
$\lrelu$ activation function.

We tested the two techniques on various architectures (\lenet, \kerasnet, \resnet, \alexnet) and datasets
(\fashionmnist, \cifar, \ilsvrc), comparing results with $\lrelu$ and single base activation functions.

In all our experiments, the techniques proposed in this paper achieved the best performance and the combined activation
functions learned using our approaches usually outperform the corresponding base components. The effectiveness
of the proposed techniques is further proved by the increase in performance achieved using networks
with different depths and architectures.

In our opinion, an interesting extension of this work would be to analyze other sets ($\Set{F}$) of base functions.

\ifCLASSOPTIONcompsoc
  % The Computer Society usually uses the plural form
  \section*{Acknowledgments}
\else
  % regular IEEE prefers the singular form
  \section*{Acknowledgment}
\fi

The authors would like to thank Enrico Deusebio for the many insightful 
discussions.

% trigger a \newpage just before the given reference
% number - used to balance the columns on the last page
% adjust value as needed - may need to be readjusted if
% the document is modified later
%\IEEEtriggeratref{8}
% The "triggered" command can be changed if desired:
%\IEEEtriggercmd{\enlargethispage{-5in}}

% references section

% can use a bibliography generated by BibTeX as a .bbl file
% BibTeX documentation can be easily obtained at:
% http://mirror.ctan.org/biblio/bibtex/contrib/doc/
% The IEEEtran BibTeX style support page is at:
% http://www.michaelshell.org/tex/ieeetran/bibtex/
\bibliographystyle{IEEEtran}

% argument is your BibTeX string definitions and bibliography database(s)
\bibliography{IEEEfull,bibliography.bib}
%
% <OR> manually copy in the resultant .bbl file
% set second argument of \begin to the number of references
% (used to reserve space for the reference number labels box)
% \begin{thebibliography}{1}

% \bibitem{IEEEhowto:kopka}
% H.~Kopka and P.~W. Daly, \emph{A Guide to \LaTeX}, 3rd~ed.\hskip 1em plus
%   0.5em minus 0.4em\relax Harlow, England: Addison-Wesley, 1999.

% \end{thebibliography}

% that's all folks
\end{document}